\documentclass[letterpaper]{article} 
\usepackage{aaai2026}  
\usepackage{times}     
\usepackage{helvet}    
\usepackage{courier}   
\usepackage[hyphens]{url}  
\usepackage{graphicx}  
\usepackage{natbib}
\usepackage{caption}
\usepackage{algorithm}        
\usepackage{algpseudocode}    

\usepackage{amssymb}
\usepackage{amsfonts}
\usepackage{amsmath}
\usepackage{amsthm}           
\newtheorem{theorem}{Theorem}

\usepackage{booktabs}         
\usepackage{multirow}
\usepackage{makecell}
\usepackage{microtype}        
\usepackage{listings}         
\usepackage{enumitem}         

\usepackage{xcolor}

\title{EVLM: Self-Reflective Multimodal Reasoning and KTO Alignment for Cross-Dimensional Visual Editing}
\author{
    Umar Khalid\textsuperscript{\rm 1*},
    Kashif Munir\textsuperscript{\rm 1*},
    Hasan Iqbal\textsuperscript{\rm 2},
    Azib Farooq\textsuperscript{\rm 3},
    Jing Hua\textsuperscript{\rm 2},
    Nazanin Rahnavard\textsuperscript{\rm 4},
    Chen Chen\textsuperscript{\rm 4},
    Victor Zhu\textsuperscript{\rm 1},
    Zhengping Ji\textsuperscript{\rm 1}
}

\affiliations{
    \textsuperscript{\rm 1}Axon, USA\\
    \textsuperscript{\rm 2}Wayne State University, USA\\
    \textsuperscript{\rm 3}Miami University, USA\\
    \textsuperscript{\rm 4}University of Central Florida, USA\\[4pt]
    \texttt{
        ukhalid@axon.com, kmunir@axon.com, 
        hasan.iqbal1292@gmail.com, azibfarooq10@gmail.com,\\
        jinghua@wayne.edu, nazanin.rahnavard@ucf.edu, chen.chen@crcv.ucf.edu,\\
        vzhu@axon.com, zji@axon.com
    }
}

\begin{document}
\maketitle
\begingroup
\renewcommand\thefootnote{}
\setlength{\parindent}{0pt}
\footnotetext{*Corresponding authors.}
\addtocounter{footnote}{-1}
\endgroup

\begin{abstract}
Editing complex visual content from ambiguous or partially specified instructions remains a core challenge in vision–language modeling. 
Existing models can contextualize content but often fail to infer the \textbf{\underline{underlying intent}} within a reference image or scene, leading to inconsistent or misaligned edits. 
We introduce the Editing Vision–Language Model (EVLM), a system that interprets ambiguous instructions in conjunction with reference visuals to produce precise, context-aware editing prompts. 
EVLM’s key innovation is a reflective reasoning framework that translates subjective user intent into structured, actionable outputs by aligning with human-rated rationales through Reflection-Aware KL-Divergence Target Optimization (RKTO). 
By combining Chain-of-Thought (CoT) reasoning with RKTO alignment, EVLM captures fine-grained editing preferences without relying on binary supervision. 
Trained on a dataset of 30,000 CoT examples with human-annotated rationale quality, EVLM achieves substantial gains in alignment with human intent. 
Experiments across image, video, 3D, and 4D editing tasks show that EVLM generates coherent and high-quality instructions, providing a scalable foundation for multimodal editing and reasoning.
\end{abstract}

\section{Introduction}

Recent advances in text-to-image (T2I) diffusion models have enabled free-form
natural language to be transformed into photorealistic imagery with striking
fidelity~\cite{ramesh2021zero,ho2022imagen,rombach2022high}. Building on this
progress, instruction-based editors—\emph{“modify this image according to that
sentence”}—have become a natural interface for visual content
creation~\cite{brooks2023instructpix2pix,kawar2023imagic,zhang2023adding}.
However, as editing tasks become more open-ended and multimodal—combining vague
language with incomplete visual cues—current systems often fail to interpret
user intent. This raises a central challenge for multimodal reasoning:
\emph{can a model reason through ambiguity, reconciling partial textual hints
with reference visuals to infer the user’s intended edit?}

\begin{figure}[t]
    \centering
    \includegraphics[width=0.9\columnwidth]{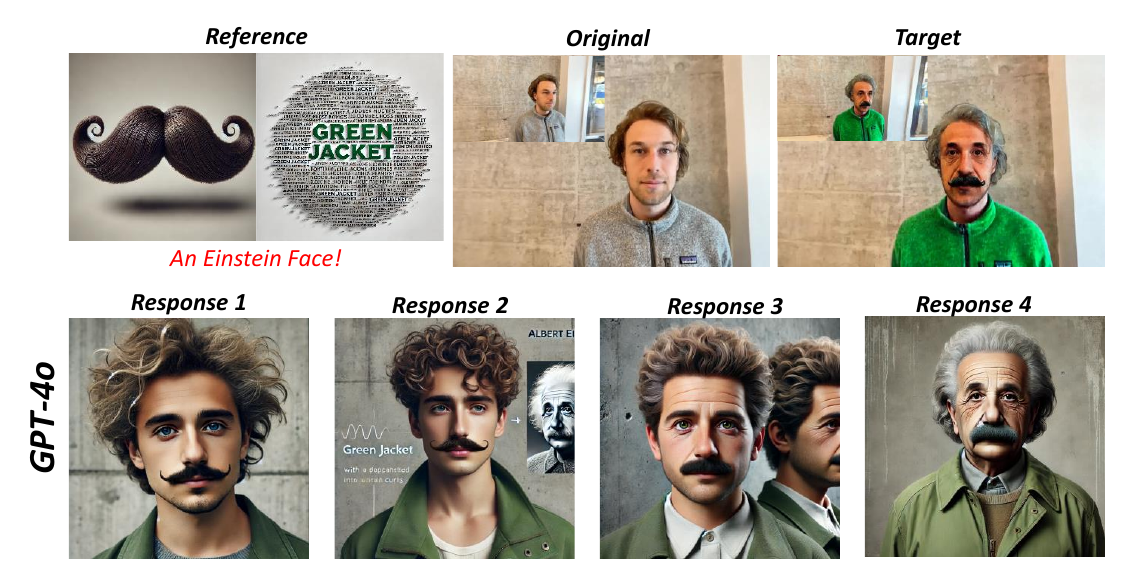}
    \caption{\footnotesize Reference image and prompt for the 3D editing task:
    "An Einstein Face!" The reference includes an image with a mustache and the
    image-with-text "Green Jacket." These were provided to GPT-4o, along with
    supporting prompts (details in {\textit{supplementary}}), to guide the
    generation of accurate editing instructions. GPT-4o encountered challenges
    integrating textual, visual, and OCR information to produce coherent
    instructions. Despite multiple attempts, DALL-E 3 guided by GPT-4o was
    unable to generate the desired edited image that fully aligns with the
    \textbf{\textit{reference}} \textbf{\textit{intent.}}}
    \label{fig:GPT}
\end{figure}

\begin{figure*}[t]
    \centering
    \includegraphics[width=0.90\textwidth]{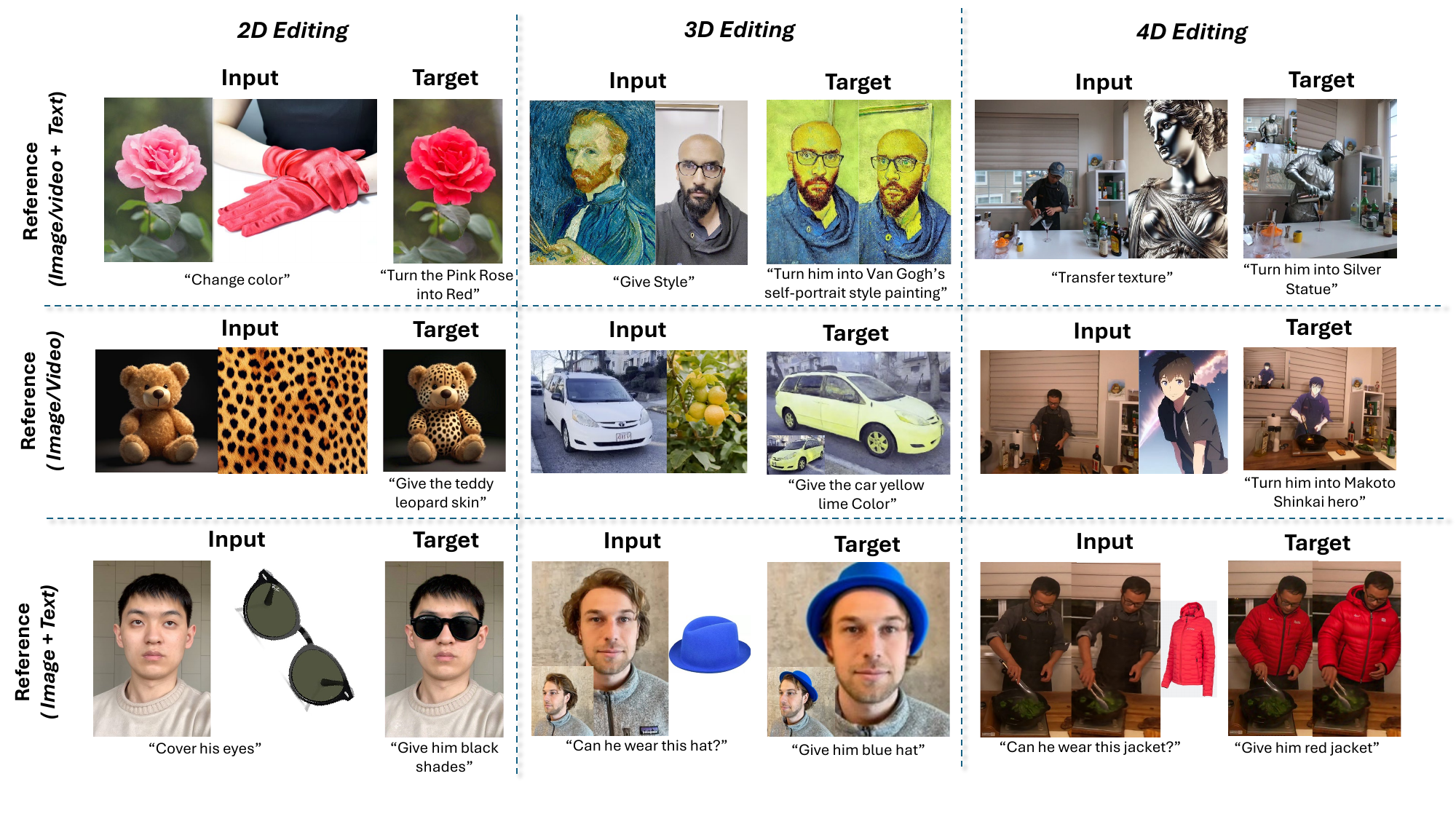}
    \vspace{-3mm}
    \caption{\footnotesize
    \textbf{EVLM enables editing across 2D, 3D, and 4D tasks.}
    Given a reference image, video, or text instruction, EVLM generates precise
    and context-aware editing transformations. Examples include color and style
    modifications in 2D and 3D, and texture or dynamic edits in 4D scenarios.
    These results highlight EVLM's multimodal understanding of spatial,
    temporal, and semantic cues for complex visual editing.}
    \vspace{-3mm}
    \label{fig:teaser}
\end{figure*}

\paragraph{Motivation.}
Prompt-refinement pipelines that attach a large language model to a diffusion
decoder~\cite{ge2024seed,zhao2024ultraedit} perform well when instructions are
explicit but lack \emph{reflective reasoning}. Given an instruction such as
``\emph{Give it an Einstein face!}'' with only weak visual hints
(Fig.~\ref{fig:GPT}), these systems cannot justify \emph{why} a specific region
—such as the jacket—should inherit a pattern or texture. Two factors contribute
to this limitation: (i) existing datasets contain only terse text–edit pairs,
providing no supervision for reasoning, and (ii) preference optimization with
PPO~\cite{schulman2017proximal} captures only coarse binary rewards, which are
unsuited for subjective, multi-solution editing tasks where subtle human
preferences matter.

We introduce the \textbf{Editing Vision–Language Model (EVLM)}, a multimodal
reasoning framework that interprets ambiguous editing instructions by combining
textual, visual, and spatial cues. EVLM ingests diverse references—such as
images, video clips, depth maps, or text—and outputs concise, disambiguated
instructions together with target masks or object indices suitable for
downstream visual editors.

Training proceeds in two complementary stages:
(1) construction of the \textbf{\textsc{Reflective-Edit}} dataset of 30k
multimodal examples, where GPT-4o generates chain-of-thought rationales that
are rated by human annotators as \emph{desired} or \emph{non-desired}; and
(2) alternating phases of \emph{Reflective Supervised Fine-Tuning} (SFT) and
\textbf{Reflection-Aware KL-Divergence Target Optimization (RKTO)}, which align
both the final instructions and the reflective reasoning process with human
preferences. Together, these components enable EVLM not only to imitate
editing instructions but also to reason reflectively about user intent with
interpretable internal logic as illustrated in
Fig.~\ref{fig:teaser}.
Our main contributions are:
\begin{enumerate}
    \item \textbf{EVLM}, a vision–language model capable of reflective
    multimodal reasoning for context-aware and interpretable editing across
    image, video, 3D, and 4D domains.
    \item \textbf{\textsc{Reflective-Edit}}, a 30k-example chain-of-thought
    dataset with human preference annotations designed to teach reflective
    reasoning for editing tasks.
    \item \textbf{RKTO}, a preference-alignment framework that extends
    KL-divergence target optimization to jointly align instruction
    effectiveness and reflection quality, providing richer and more stable
    feedback than PPO-based methods.
\end{enumerate}


\section{Related Work}

\footnotesize 
\setlength{\parskip}{2pt} 

\paragraph{Reflection and Alignment in Multimodal Models.}
Recent work explores reflection in language and vision-language models to enable post-hoc self-correction and alignment with human intent.
Methods use either external feedback (\textit{e.g.}, execution traces, expert critiques)~\cite{NEURIPS2023_1b44b878, chen2024teaching} or internal self-evaluation~\cite{madaan2024self, li-etal-2024-hindsight, weng2023large}, though reliability remains task-dependent~\cite{huang2023large}.
Parallel efforts in vision-language models apply chain-of-thought reasoning to domains such as math~\cite{lu2023mathvista, wang2024measuring}, scientific QA~\cite{lu2022learn}, and chart understanding~\cite{zhang2024tinychart}.
Preference-alignment approaches like DPO~\cite{ouali2024clip, sun2023aligning} and PPO~\cite{yu2024rlaif} guide model fine-tuning, while iterative DPO~\cite{sun2024easy, pang2024iterative} improves step-by-step reasoning.
More recently, KTO~\cite{ethayarajh2024kto} aligns model outputs to a soft reference distribution via KL-divergence, effectively modeling nuanced human preferences for subjective generation tasks such as visual editing.

\begin{figure*}[t]
    \centering
    \includegraphics[width=0.90\textwidth]{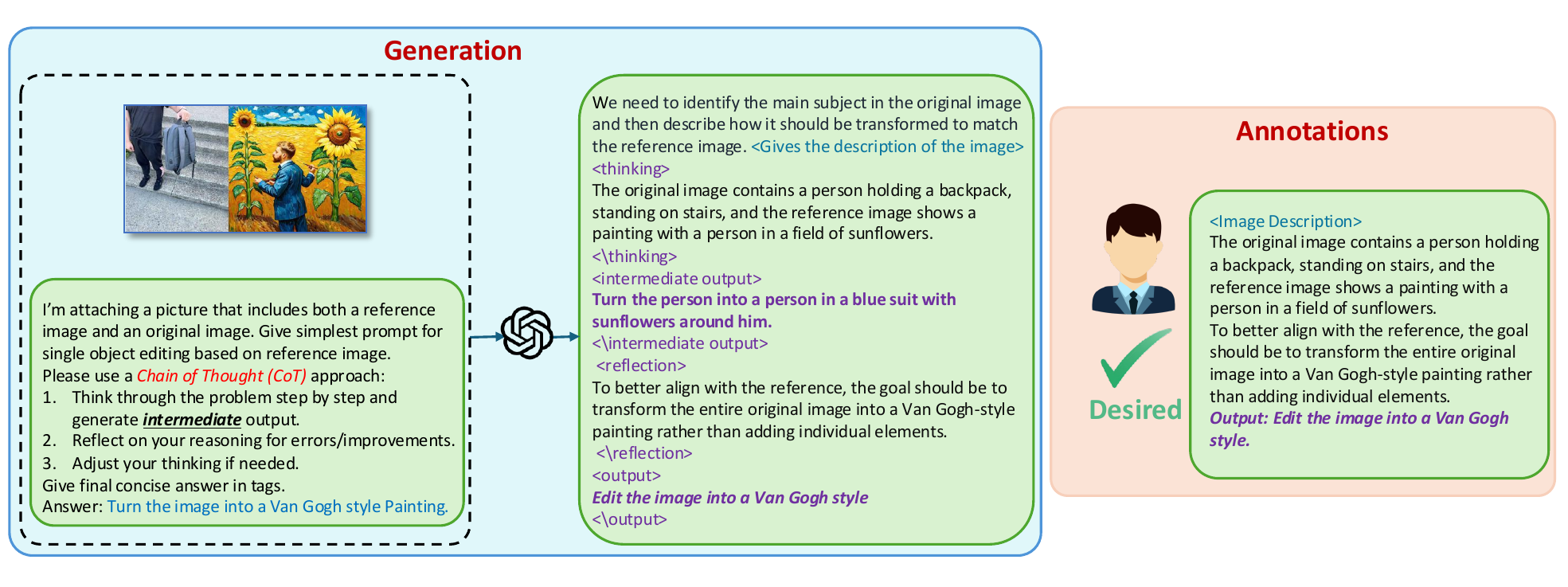}
    \vspace{-3mm}
    \caption{\footnotesize
    \textbf{Overview of our data preparation pipeline.}
    Given a reference and target image, GPT-4o produces a structured chain-of-thought rationale through initial, intermediate, and reflective reasoning.
    Only the reflective and final outputs are used to construct RKTO training data, with human annotators providing “desired” labels when reasoning aligns with intended edits.}
    \label{fig:EVLm-Pipeline}
    \vspace{-2mm}
\end{figure*}

\paragraph{Diffusion-Based Editing Across Dimensions.}
Diffusion-based models enable text-driven editing across modalities, spanning 2D~\cite{ruiz2023dreambooth, pnp, p2p, sdedit, brooks2023instructpix2pix}, 3D~\cite{kamata2023instruct, instructnerf2023, dong2024vica}, and 4D~\cite{zhuang2023dreameditor, shao2023control4d} domains.
Extensions to videos include Tune-A-Video~\cite{wu2023tune}, Make-A-Video~\cite{singer2023makeavideo}, and MagicVideo~\cite{zhou2022magicvideo}, which leverage spatio-temporal attention to maintain consistency.
Prompt-level control methods such as Prompt-to-Prompt~\cite{p2p} and Plug-and-Play~\cite{pnp} refine local edits, while 3D and 4D variants adapt text-to-image priors (e.g., IP2P) for NeRFs and dynamic scenes.
Our work builds on these foundations by integrating reflective reasoning and alignment into the diffusion pipeline, yielding flexible editing across 2D, 3D, and 4D modalities.
\normalsize 
\section{Approach}
\label{sec:approach}

The Editing Vision--Language Model (EVLM) converts multimodal editing intent—
including images, video, and text—into concise and interpretable editing
instructions accompanied by spatial masks. 
Training proceeds in three stages:
(1) construction of a reflective dataset using GPT-4o,
(2) supervised fine-tuning (SFT) of a Qwen2-VL-7B backbone for structured
reasoning, and
(3) reflection-aware KL-Divergence Target Optimization (RKTO) to align both
the generated instructions and the underlying reflective reasoning with human
preferences. (Notations and full derivations are provided in the Supplementary.)

\subsection{\textsc{Reflective-Edit} Dataset}
We construct the \textsc{Reflective-Edit} dataset containing approximately
30{,}000 multimodal examples. Each example contains an input, one or more
reference items (image, video, or text), and a GPT-4o-generated chain-of-thought
(CoT) trace segmented into \texttt{<thinking>}, \texttt{<intermediate>},
\texttt{<reflection>}, and \texttt{<output>} sections. Human annotators label each trace
as \emph{desired} or \emph{non-desired}, providing graded preference supervision
for both instruction quality and reflective reasoning. Figure~\ref{fig:EVLm-Pipeline}
summarizes the data-preparation pipeline, including CoT generation and human annotation.

\subsection{Supervised Fine-Tuning (SFT)}
EVLM first learns to reproduce GPT-4o CoT traces using teacher forcing. For a
tuple $(\mathcal{V},u,y_{1:T})$, where $\mathcal{V}$ denotes the visual
context, $u$ the textual prompt, and $y_{1:T}$ the target token sequence
(which may include mask tokens yielding a discrete mask $\hat m$), the
SFT objective is
\[
\mathcal{L}_{\mathrm{SFT}}
 = -\sum_{t=1}^{T}\log p_\theta(y_t\mid y_{<t},u,\mathcal{V}),
\]
which stabilizes token-level generation and provides the reference snapshot
$\rho_{\mathrm{ref}}$ used in subsequent alignment. When available, the dataset
also provides a reference mask $m_{\mathrm{ref}}$ for downstream evaluation.

\subsection{Reflection-Aware KTO (RKTO)}
We extend KL-Divergence Target Optimization (KTO) by incorporating a reflection
reward that encourages concise, consistent reflective reasoning in addition to
instruction fidelity. For a tuple $(\mathcal{V},u,y_{\mathrm{pref}},m_{\mathrm{ref}})$,
the per-sample RKTO objective is
\begin{equation}
\begin{split}
\mathcal{L}_{\mathrm{RKTO}} = \mathbb{E}\Big[\, &
  w(s_\phi-\eta_0)\,
    R_{\mathrm{eff}}(\hat y,\hat m; y_{\mathrm{pref}}, m_{\mathrm{ref}})\\
&\qquad\qquad\qquad
  +\;\lambda_{\mathrm{ref}}\,R_{\mathrm{reflect}}(r_{\mathrm{refl}})\,\Big],
\end{split}
\label{eq:rkto_loss}
\end{equation}
where
\begin{equation}
\begin{split}
s_\phi &= 
\log\!\frac{\rho_\phi(y_{\mathrm{pref}}\mid u,\mathcal{V})}
              {\rho_{\mathrm{ref}}(y_{\mathrm{pref}}\mid u,\mathcal{V})},\\[3pt]
\eta_0 &= 
\mathrm{KL}\!\Big(
\rho_\phi(\cdot\mid u,\mathcal{V})
\;\Big\|\;
\rho_{\mathrm{ref}}(\cdot\mid u,\mathcal{V})
\Big).
\end{split}
\end{equation}

Here we clarify key terms used above: \(\rho_{\mathrm{pref}}\) denotes the empirical human-preferred distribution estimated from annotator labels; \(\hat m\) denotes the model-decoded mask (from mask tokens) and \(m_{\mathrm{ref}}\) the dataset reference mask; IoU is non-differentiable and is handled via REINFORCE with a batch baseline; \(w(\cdot)\) is \(\mathrm{softplus}\) followed by clipping to \([0,w_{\max}]\); cosine similarities are rescaled to \([0,1]\) by \(\widetilde{\cos}(a,b)=(1+\cos(a,b))/2\); and \(r_{\mathrm{refl}}\) denotes the reflection segment from the CoT trace.  

\begin{figure*}
    \centering
    \includegraphics[width=0.9\linewidth]{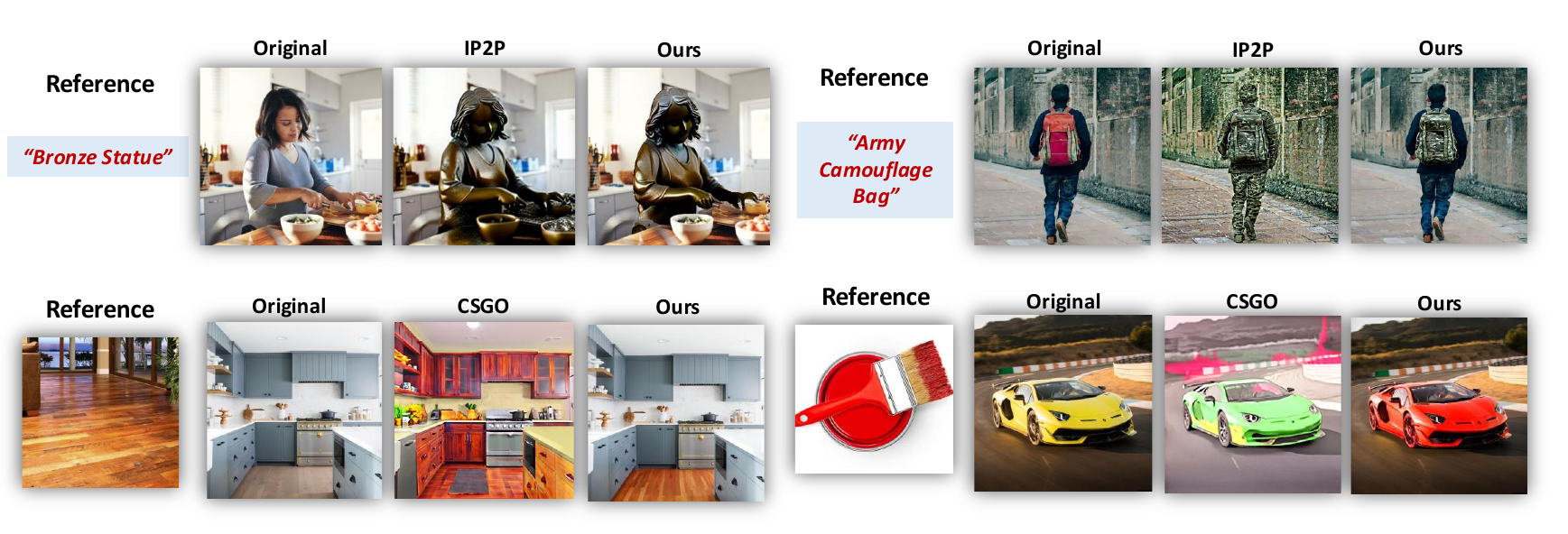}
    \caption{\textbf{Image Editing Results Using Reference Images or Text Prompts}. The first row demonstrates EVLM's ability to refine vague textual prompts into precise editing instructions and to produce masks for targeted edits. The second row compares EVLM to an image-based editing baseline, illustrating improved alignment to reference-guided transformations.}
    \label{fig:image_results}
    \vspace{-5mm}
\end{figure*}

\paragraph{Instruction Effectiveness.}
When a reference mask \(m_{\mathrm{ref}}\) is available, the instruction-effectiveness
reward combines semantic and spatial alignment:
\begin{equation}
\begin{split}
R_{\mathrm{eff}}(\hat y,\hat m; y_{\mathrm{pref}}, m_{\mathrm{ref}})=\;&
\alpha\,\widetilde{\cos}\!\big(e(\hat y),e(y_{\mathrm{pref}})\big)\\
&+(1-\alpha)\,\mathrm{IoU}(\hat m,m_{\mathrm{ref}}),
\end{split}
\end{equation}
where \(e(\cdot)\) is the Qwen2-VL encoder, 
\(\widetilde{\cos}(a,b)=(1+\cos(a,b))/2\) rescales cosine similarity to \([0,1]\),
and \(\mathrm{IoU}\) denotes intersection-over-union between predicted and
reference masks. IoU is non-differentiable and is optimized via REINFORCE
with a batch baseline.

\paragraph{Reflection Reward.}
The reflection-quality reward is
\[
\begin{aligned}
R_{\mathrm{reflect}}(r_{\mathrm{refl}}) =\;&
\beta_1\,\widetilde{\cos}\!\big(e(r_{\mathrm{refl}}),e(y_{\mathrm{pref}})\big)\\
&+\beta_2\,\exp\!\big(-\gamma\,\mathrm{len}(r_{\mathrm{refl}})\big)\\
&+\beta_3\big(1-\mathrm{KL}(p_{\mathrm{int}}\|p_{\mathrm{refl}})\big),
\end{aligned}
\]
where \(p_{\mathrm{int}}\) and \(p_{\mathrm{refl}}\) are predictive token distributions
for the \texttt{<intermediate>} and \texttt{<reflection>} segments, and the \(\beta\)
coefficients sum to one. The reflection loss is \(\mathcal{L}_{\mathrm{reflect}}=1-R_{\mathrm{reflect}}\).

\subsection{Gradient Estimation and Optimization}
\label{sec:grad_est}

RKTO computes gradients at the sequence level, combining differentiable
token-likelihood terms with policy gradients for non-differentiable components
such as IoU. For a batch of $B$ preference examples, the weighted objective
involves normalized log-ratio and baseline statistics given by
\begin{equation}
\begin{split}
s_\phi &=
\log\!\frac{\rho_\phi(y_{\mathrm{pref}}\mid u,\mathcal{V})}
          {\rho_{\mathrm{ref}}(y_{\mathrm{pref}}\mid u,\mathcal{V})},\\[2pt]
\eta_0 &=
\mathrm{KL}\!\Big(\rho_\phi(\cdot\mid u,\mathcal{V})
\;\Big\|\;
\rho_{\mathrm{ref}}(\cdot\mid u,\mathcal{V})\Big).
\end{split}
\end{equation}

The instruction-effectiveness reward, which combines semantic and spatial
consistency, is formulated as
\begin{equation}
\begin{split}
R_{\mathrm{eff}}(\hat y,\hat m; y_{\mathrm{pref}},m_{\mathrm{ref}})=&
\;\alpha\,\widetilde{\cos}\!\big(e(\hat y),e(y_{\mathrm{pref}})\big)\\
&+(1-\alpha)\,\mathrm{IoU}(\hat m,m_{\mathrm{ref}}),
\end{split}
\end{equation}
where $\widetilde{\cos}(a,b)=(1+\cos(a,b))/2$ rescales the cosine similarity
to $[0,1]$, $e(\cdot)$ is the text-encoder embedding, and $\mathrm{IoU}$ is the
intersection-over-union between predicted and reference masks.

For conciseness, let
\[
c_i=w(\hat s_\phi^{(i)}-\hat\eta_0), \qquad
d_i=c_i(1-\alpha)(\mathrm{IoU}^{(i)}-b_{\mathrm{IoU}}),
\]
where $b_{\mathrm{IoU}}$ denotes a batch baseline.
The resulting gradient estimator can then be expressed as
\begin{equation}
\label{eq:grad_est_final}
\begin{aligned}
\nabla_\phi\mathcal{L}_{\mathrm{RKTO}}
\approx\;&\frac{1}{B}\sum_{i=1}^{B} c_i\,R_{\mathrm{eff}}^{(i)}\,
  \nabla_\phi\log\rho_\phi\!\big(y_{\mathrm{pref}}^{(i)}\mid x^{(i)}\big)\\
&+\lambda_{\mathrm{ref}}\nabla_\phi\mathcal{L}_{\mathrm{reflect}}\\
&+\frac{1}{B}\sum_{i=1}^{B} d_i\,
  \nabla_\phi\log\rho_\phi\!\big(\hat m^{(i)}\mid x^{(i)}\big).
\end{aligned}
\end{equation}

The first term represents the importance-weighted policy gradient for
the preferred textual response, the second corresponds to the differentiable
reflection component, and the third applies a REINFORCE correction for the
non-differentiable IoU reward on discrete masks.  
Variance is reduced through clipping of $w(\cdot)$ to $[0,w_{\max}]$,
centering of rewards using batch baselines, and a smaller learning rate for
RKTO relative to SFT.

SFT and RKTO are alternated during training. 
SFT provides a stable reference policy and structured reasoning capacity, 
while RKTO progressively aligns the model with human-preferred outputs and reflective reasoning. 
Convergence is observed once validation accuracy and reflection-quality metrics stabilize.

Optimizing Eq.~\ref{eq:rkto_loss} jointly reduces the divergence between the model 
and human-preferred distributions for both outputs and reflections. 
Under standard regularity conditions—bounded rewards, monotonic $w(\cdot)$, 
and sufficiently small learning steps—gradient updates that minimize the RKTO objective 
also decrease a composite KL divergence between model and preference distributions 
(see Supplementary for the formal theorem and proof).

\begin{table*}[t]
\centering
\caption{Accuracy comparison across different evaluators on the \textsc{Reflective-Edit} benchmark. The last column shows the average accuracy across all evaluators.}
\label{tab:accuracy_comparison}
\footnotesize 
\resizebox{0.8\textwidth}{!}{
\begin{tabular}{lcccccc}
\toprule
\textbf{Model} &
\begin{tabular}[c]{@{}c@{}}\textbf{Gemini}\\\textbf{Pro 1.5}\end{tabular} &
\begin{tabular}[c]{@{}c@{}}\textbf{LLAMA}\\\textbf{405B}\end{tabular} &
\textbf{GPT-4o} &
\begin{tabular}[c]{@{}c@{}}\textbf{Claude 3.5}\\\textbf{Sonnet}\end{tabular} &
\begin{tabular}[c]{@{}c@{}}\textbf{Human}\\\textbf{Evaluators}\end{tabular} &
\begin{tabular}[c]{@{}c@{}}\textbf{Avg.}\\\textbf{Accuracy}\end{tabular} \\
\midrule
mPLUG-Owl~\cite{ye2023mplug}        & 44.8 & 45.1 & 43.3 & 42.7 & 49.0 & 44.8 \\
mPLUG-Owl2~\cite{ye2024mplug}       & 50.2 & 51.6 & 51.3 & 49.5 & 51.2 & 50.8 \\
LLAVA~\cite{liu2024visual}          & 48.3 & 46.8 & 47.2 & 45.9 & 43.5 & 46.3 \\
MiniGPT-4~\cite{zhu2023minigpt}     & 43.7 & 44.1 & 42.5 & 45.3 & 47.6 & 44.6 \\
CogVLM~\cite{wang2023cogvlm}        & 41.4 & 42.2 & 40.7 & 39.6 & 55.2 & 43.8 \\
InstructBLIP~\cite{instructblip}    & 47.6 & 46.9 & 48.4 & 47.2 & 56.3 & 49.2 \\
Qwen-VL~\cite{bai2023qwen}          & 49.3 & 50.4 & 48.6 & 47.8 & 52.4 & 49.7 \\
LLAMA-3.2-11B~\cite{dubey2024llama} & 55.5 & 57.2 & 56.3 & 54.8 & 60.1 & 56.7 \\
LLAMA-3.2-90B~\cite{dubey2024llama} & 64.7 & 63.5 & 66.1 & 64.3 & 58.9 & 63.5 \\
\midrule
\textbf{EVLM-RKTO} & \textbf{95.4} & \textbf{94.8} & \textbf{96.2} & \textbf{95.1} & \textbf{94.1} & \textbf{95.1} \\
\bottomrule
\end{tabular}%
} 
\end{table*}

\begin{table}[b]
\centering
\footnotesize
\caption{Cross-benchmark performance (\%) on multimodal reasoning and visual question answering benchmarks.}
\label{tab:GQA}
\resizebox{\columnwidth}{!}{%
\begin{tabular}{lcccc}
\toprule
\textbf{Benchmark} &
\begin{tabular}[c]{@{}c@{}}\textbf{LLAMA}\\\textbf{11B}\end{tabular} &
\begin{tabular}[c]{@{}c@{}}\textbf{Qwen2}\\\textbf{7B}\end{tabular} &
\begin{tabular}[c]{@{}c@{}}\textbf{LLAVA 1.6}\\\textbf{(Vicuna 7B)}\end{tabular} &
\begin{tabular}[c]{@{}c@{}}\textbf{EVLM}\\\textbf{(RKTO)}\end{tabular} \\
\midrule
MMMU (val, CoT)           & 50.7 & \textbf{54.1} & 35.8 & 53.0 \\
MMMU-Pro (Vision)         & 33.0 & 43.5 & -- & \textbf{43.8} \\
MathVista (testmini)      & 51.5 & 58.2 & 34.6 & \textbf{59.1} \\
ChartQA (test, CoT)       & \textbf{83.4} & 83.0 & -- & 82.8 \\
AI2 Diagram (test)        & \textbf{91.1} & 83.0 & -- & 83.0 \\
DocVQA (test)             & 88.4 & \textbf{94.5} & -- & 93.1 \\
VQAv2 (test)              & 75.2 & -- & \textbf{81.8} & 77.4 \\
\bottomrule
\end{tabular}%
}
\end{table}

\section{Experiments}
\label{sec:expr}

\subsection{Model Architecture}
Our model builds on the Qwen2-VL-7B architecture~\cite{wang2024qwen2}, which combines a 675M-parameter vision encoder with a 7.6B-parameter language model to strengthen multimodal reasoning.  
Naive Dynamic Resolution~\cite{dehghani2024patch} allows variable-size visual inputs by dynamically tokenizing images without absolute position embeddings.  
We further incorporate 2D-RoPE and Multimodal Rotary Position Embedding (M-RoPE)~\cite{wang2024qwen2} to encode spatial and temporal relations across text, images, and video, enabling precise positional reasoning.


\subsection{Evaluation on \textsc{Reflective-Edit}}

Since generating editing instructions is inherently subjective and often depends on human preferences, we evaluate \textsc{EVLM} using two complementary criteria. 

\begin{enumerate}[leftmargin=*, topsep=2pt]
    \item \textbf{Human-labeled Evaluation.} 
    We generate editing instructions and ask human annotators to assign binary \texttt{YES}/\texttt{NO} labels, indicating whether the generated instruction semantically aligns with the intended edit. These labels serve as the ground truth for evaluating \textsc{EVLM} and comparing it against other baselines.

    \item \textbf{LLM-based Evaluation.} 
    To further ensure scalability and consistency, we employ large language models (LLMs) as automated judges. Specifically, we use four evaluators—\texttt{Gemini~Pro~1.5}, \texttt{LLaMA~405B}, \texttt{GPT-4o}, and \texttt{Claude~3.5~Sonnet}—to provide binary \texttt{YES}/\texttt{NO} judgments on 3{,}000 unseen examples from the \textsc{Reflective-Edit} benchmark. Each evaluator marks \texttt{YES} if the generated instruction semantically matches the human reference.
\end{enumerate}

As reported in Table~\ref{tab:accuracy_comparison}, \textsc{EVLM-RKTO} outperforms all baselines, achieving an average accuracy of \textbf{95.1\%}. This demonstrates the effectiveness of reflection-aware fine-tuning for improving instruction alignment.
\paragraph{Evaluation protocol.}
We (i) fix evaluator prompts (see Supplementary), (ii) sample 3{,}000 stratified examples balanced across edit types, and (iii) compute per-evaluator agreement statistics and bootstrap confidence intervals.



\subsection{General Visual Understanding}
This section assesses EVLM's ability to retain generalization across multimodal benchmarks.  
As summarized in Table~\ref{tab:GQA}, EVLM achieves strong performance despite having fewer parameters than LLAMA-11B.  
It attains the highest scores on MathVista (59.1) and MMMU-Pro Vision (43.8), while matching or surpassing LLAMA-11B on DocVQA (93.1) and VQAv2.  
These results confirm that reflection-aware training enhances alignment and reasoning quality without compromising generalization.

\subsection{Cross-Dimensional Visual Editing}
EVLM serves as a plug-and-play module that refines textual prompts and identifies target objects for diffusion-based editors.  
When integrated with IP2P~\cite{brooks2023instructpix2pix}, it outperforms both text-only (IP2P) and image-based (CSGO) approaches in precision and controllability (Fig.~\ref{fig:image_results}).  
For video editing, we apply the Tune-A-Video recipe~\cite{wu2023tune}, where EVLM improves temporal consistency compared with Any-V2V~\cite{ku2024anyv2v}.  
\begin{figure*}[htb]
\centering
\includegraphics[width=0.85\linewidth]{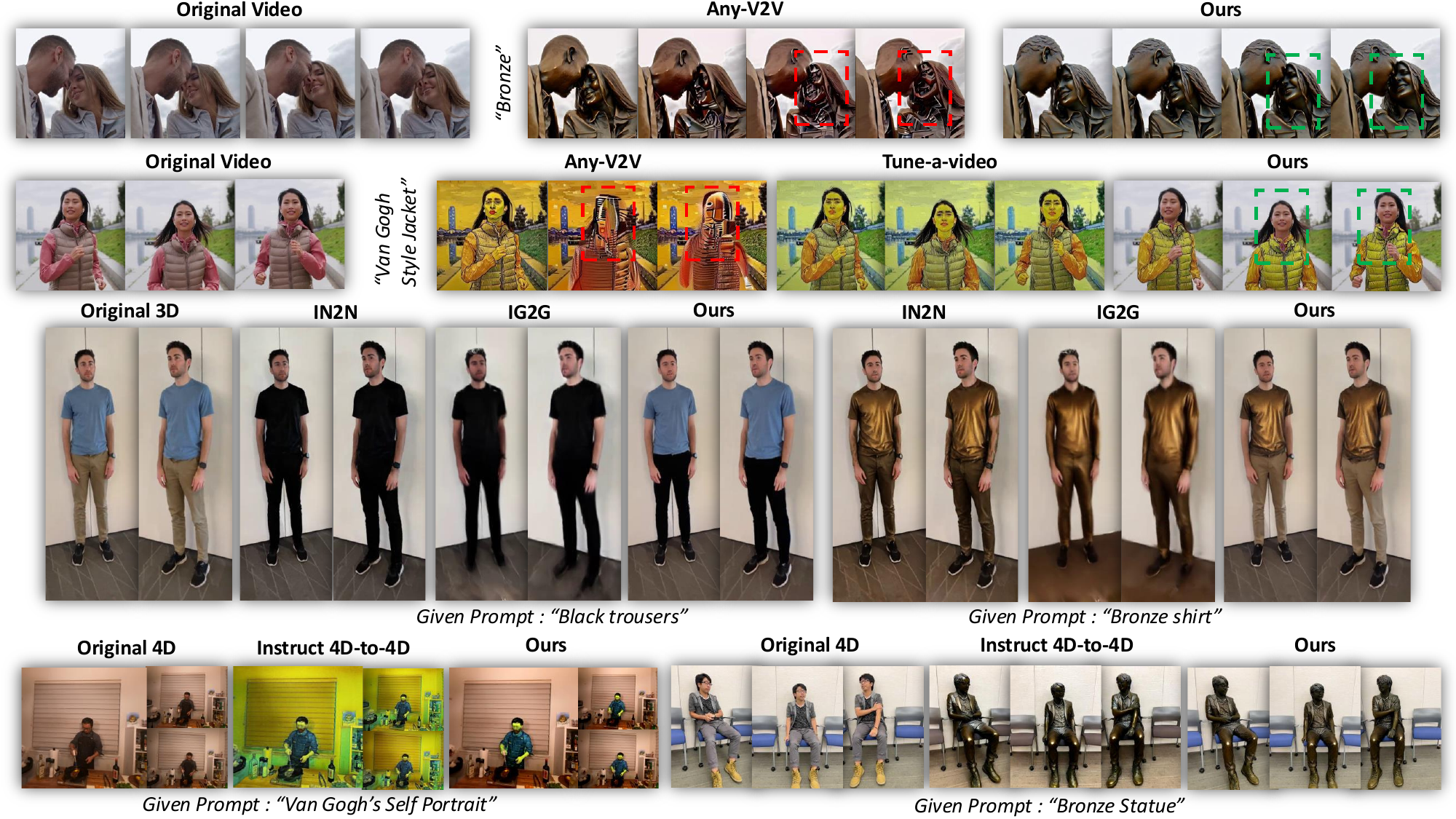}
\caption{\textbf{Text-based editing across video, 3D, and 4D tasks.}
EVLM + IP2P surpasses baselines including Any-V2V, Tune-A-Video, IN2N, and IG2G.  
It delivers improved style consistency (``Bronze,'' ``Van Gogh Style'') and stable transformations across frames.}
\label{fig:comp_results}
\vspace{-3mm}
\end{figure*}
For 3D editing, we employ 3D Gaussian Splatting~\cite{kerbl20233d} for scene reconstruction and compare against IN2N~\cite{haque2023instruct} and IG2G on the IN2N and 3DEgo datasets.  
In 4D editing, we evaluate on DyCheck, HyperNeRF, and DyNeRF/N3DV datasets using the frameworks of~\cite{wu20244d,mou2024instruct}.  
Across these tasks (Fig.~\ref{fig:comp_results}), EVLM-guided edits exhibit superior spatial consistency, color coherence, and adherence to user intent.

\begin{table}[b]
\vspace{-3mm}
\centering
\caption{Performance of \textsc{EVLM-RKTO} under varying RL train-set sizes. }
\label{tab:trainset_size}
\footnotesize 
\setlength{\tabcolsep}{4pt} 
\resizebox{0.85\columnwidth}{!}{
\begin{tabular}{lcc}
\toprule
\textbf{Model / Components} & \textbf{Dataset Size} & \textbf{Accuracy (\%)} \\
\midrule
Qwen-2.5-VL-7B & -- & 52.4 \\
EVLM-KTO & 5K & 56.1 \\
EVLM-RKTO & 5K & 59.1 \\
EVLM-RKTO & 7K & 65.7 \\
EVLM-RKTO & 15K & 78.2 \\
EVLM-RKTO & 30K & \textbf{94.1} \\
\bottomrule
\end{tabular}}
\vspace{-3mm}
\end{table}
\subsection{Ablation Studies}

\paragraph{RL Train-Set Size.}
We analyze the sensitivity of \textsc{EVLM-RKTO} to the size of the reinforcement learning (RL) training set. 
Specifically, we train models on progressively smaller subsets of our original 30K dataset, using samples of 5K, 7K, and 15K. 
As shown in Table~\ref{tab:trainset_size}, \textsc{RKTO} consistently benefits from larger training data, exhibiting steady improvements in alignment and reasoning performance. 
Remarkably, even when trained with only 5K examples, \textsc{EVLM-RKTO} significantly surpasses both the base \textsc{Qwen-2.5-VL-7B} model and the standard \textsc{EVLM-KTO} variant, achieving absolute gains of 6.7\% and 3.0\%, respectively. 
These results indicate that incorporating self-reflection within RL training substantially enhances reasoning robustness and data efficiency, enabling strong performance even under limited supervision.

\paragraph{Effectiveness of Reflection-Aware RL Components.}

We observe that omitting any specific component of the Self-Reflection Reinforcement Learning (SRL) framework leads to performance degradation. 
In particular, removing the \textit{Effectiveness Reward} ($R_{\text{eff}}$) results in a drop in average accuracy from 94.1\% to 89.2\%, indicating that the model critically depends on reward signals that explicitly evaluate the quality of reflective responses to achieve optimal reasoning. 
Similarly, excluding the \textit{Reflection Reward} ($R_{\text{reflect}}$) reduces the performance from 94.1\% to 85.3\%, suggesting that missing reflection steps can interfere with the model’s reasoning accuracy.
\begin{table}[htb!]
\centering
\caption{Comparison of different training strategies evaluated by human annotators.}
\label{tab:kto_sft_comparison}
\footnotesize
\setlength{\tabcolsep}{4pt}
\resizebox{0.85\columnwidth}{!}{
\begin{tabular}{lc}
\toprule
\textbf{Method} & \textbf{Accuracy (\%)} \\
\midrule
EVLM-KTO & 75.1 \\
EVLM-KTO (2-step thinking) & 76.3 \\
EVLM-SFT & 67.2 \\
EVLM-SFT (2-step thinking) & 65.1 \\
\midrule
\textbf{Ours (EVLM-RKTO)} & \textbf{94.1} \\
\bottomrule
\end{tabular}}
\end{table}
\paragraph{Comparison with Alternative Training Objectives.}
To further validate the effectiveness of our reflection-aware fine-tuning, we train \textsc{EVLM} under standard KTO and SFT objectives, as well as the two-step thinking variants where the model generates \texttt{<think>}$\ldots$\texttt{</think>} steps before generating the final editing instruction.

Table~\ref{tab:kto_sft_comparison} summarizes the results, showing that our proposed \textsc{EVLM-RKTO} substantially outperforms all alternatives.

\section{Conclusion}
\label{sec:conclusion}

We present EVLM, a reflection-enabled vision–language model that resolves ambiguity in multimodal editing by combining Chain-of-Thought supervision with Reflection-aware KL-Divergence Target Optimization (RKTO). Trained on a curated \textsc{Reflective-Edit} corpus of 30{,}000 CoT examples and validated on 3{,}000 held-out cases, EVLM produces concise, disambiguated editing instructions and target masks that improve downstream diffusion- and NeRF-based editors. Empirically, RKTO improves both instruction fidelity and reflection quality, yielding strong gains across image, video, 3D, and 4D editing tasks.

Future work will (i) expand CoT diversity to capture broader user styles and edge cases, (ii) formalize $R_{\mathrm{reflect}}$ with additional human-in-the-loop metrics (e.g., brevity vs.\ usefulness trade-offs), and (iii) study statistical significance and robustness under noisy references. We believe EVLM's reasoning-with-reflection paradigm generalizes beyond editing to other multimodal tasks that require interpretable decisions and subjective preference alignment.

\bibliography{aaai2026}
\clearpage
\appendix
\renewcommand{\thesection}{\Alph{section}} 
\setcounter{section}{0}
\setcounter{secnumdepth}{2}                

\section{Overview}
\label{sec:overview}
Following is the content list of the supplementary material. 
\begin{itemize}[noitemsep]
  \item Sec.~\ref{sec:algorithm}: EVLM training Algorithm (SFT $\leftrightarrow$ RKTO).
  \item Sec.~\ref{sec:eval}: Evaluation Details..
  \item Sec.~\ref{sec:stats}: Evaluation statistics.
  \item Sec.~\ref{sec:notation}: Theoretical Analysis.
  \item Sec.~\ref{sec:impl}: Implementation Details.
  \item Sec.~\ref{sec:dataset}: Dataset generation, prompts, and distributions.
  \item Sec.~\ref{sec:add_results}: Additional Results.
  \item Sec.~\ref{sec:ablations}: Ablation Studies.
  \item Sec.~\ref{sec:user}: User Study
  
  \item Sec.~\ref{sec:limitations}: Limitations.
 
\end{itemize}

\section{Training Algorithm}
\label{sec:algorithm}

Algorithm~\ref{alg:sft_kto} summarizes the iterative training procedure: initial Supervised Fine-Tuning (SFT) to learn structured CoT traces, followed by alternating RKTO optimization to align outputs and reflections with human preferences.

\begin{algorithm}[h]
\caption{Iterative Training with SFT and RKTO}
\label{alg:sft_kto}
\begin{algorithmic}[1]
\Require Reference input \( R \), Original input \( O \), Dataset $\mathcal{D}=\{(x_i,y^{\text{ref}}_i,r^{\text{ref}}_i)\}$, SFT epochs $E_{\text{SFT}}$, RKTO epochs $E_{\text{RKTO}}$
\State Initialize EVLM model $\rho_\phi$ from Qwen2-VL-7B backbone
\For{\( e=1 \) \textbf{to} \(E_{\text{SFT}}\)}
  \State Minimize $\mathcal{L}_{\mathrm{SFT}} = -\sum_t \log p_\phi(y_t^{\text{ref}}\mid y_{<t}, x)$
  \State Periodically save SFT snapshot $\rho_{\mathrm{ref}}$
\EndFor
\For{\( e=1 \) \textbf{to} \(E_{\text{RKTO}}\)}
  \State Sample preference batch $\{(x_i, y^{(i)}_{\mathrm{pref}}, R^{(i)}_{\mathrm{eff}}, R^{(i)}_{\mathrm{reflect}})\}$
  \State Compute importance weights and $\widehat{\mathcal{L}}_{\mathrm{RKTO}}$ (Eq.~\ref{eq:batch_rkto_app})
  \State Update $\phi$ using the estimator in Eq.~\ref{eq:grad_est_app} (policy-gradient + reflection surrogate)
  \If{checkpoint time} save $\rho_{\mathrm{ref}}\leftarrow\rho_\phi$
  \EndIf
\EndFor
\end{algorithmic}
\end{algorithm}

Implementation notes:
\begin{itemize}[noitemsep]
  \item We cache SFT log-probs to reduce RKTO computation where possible.
  \item RKTO uses a smaller LR than SFT for stability (see Sec.~\ref{sec:impl}).
  \item Non-differentiable components (IoU) are optimized using REINFORCE with baseline centering.
\end{itemize}

\section{Evaluation Details}
\label{sec:eval}

This section reproduces the evaluator prompts and a concise summary of human annotator instructions used in the experiments (these were fixed for the 3,000-example evaluation split).

\subsection*{LLM evaluator prompt}
\begin{quote}\small
You are given: (1) the original image (and reference image/text if present), (2) a human reference editing instruction, and (3) an automatically generated editing instruction. Your task: answer \texttt{YES} if the generated instruction semantically matches the human reference (i.e., would produce a visually similar edit to the target), otherwise answer \texttt{NO}. Consider object identity, target region, and transformation details. Do not hallucinate details beyond the provided inputs. Reply with a single token: \texttt{YES} or \texttt{NO}.
\end{quote}

\subsection*{Human annotator instructions}
Annotators received the following guidelines (full guideline documents were provided to raters and released with code):
\begin{itemize}[noitemsep]
  \item Present the original image, the reference (if any), and the candidate instruction side-by-side.
  \item Ask: ``Does this instruction, when applied, match the intended edit?'' — answer \texttt{YES}/\texttt{NO}.
  \item Provide illustrative examples of \texttt{YES} and \texttt{NO} cases, emphasizing object semantics (``bag'', ``jacket''), region specificity (``left sleeve''), and transformation type (``change color'', ``add pattern'').
  \item Ambiguous cases: use majority vote among 3 raters; record per-rater judgments to compute Cohen's and Fleiss' kappa.
  \item Instruct annotators not to assume external context beyond the shown reference(s).
\end{itemize}

\section{Evaluation Statistics}
\label{sec:stats}

We describe the protocol used to compute per-evaluator accuracies, bootstrap confidence intervals (CIs), and inter-rater agreement reported in the paper.

\subsection*{Evaluation sampling}
\begin{itemize}[noitemsep]
  \item Draw a stratified evaluation set of $N=3{,}000$ examples balanced across edit types (color, style, texture, object-replacement, layout change).
  \item For each example, collect binary judgments (\texttt{YES}/\texttt{NO}) from each evaluator (LLMs and humans).
\end{itemize}

\subsection*{Bootstrap CI computation}
\begin{enumerate}[noitemsep]
  \item Draw 10{,}000 bootstrap resamples (with replacement) of size $N$ from the evaluation set.
  \item For each resample and each evaluator, compute the fraction of \texttt{YES} responses.
  \item The reported 95\% CI is the empirical 2.5th--97.5th percentile of the bootstrap distribution.
\end{enumerate}

\subsection*{Agreement measures}
\begin{itemize}[noitemsep]
  \item For human annotators, compute pairwise Cohen’s $\kappa$ and Fleiss’ $\kappa$ when groups of \>2 raters are used.
  \item For human vs. LLM agreement, compute percent agreement and Cohen’s $\kappa$ between LLM judgments and the majority human label.
\end{itemize}

\subsection*{Representative CI numbers (example)}
Table~\ref{tab:bootstrap_dummy} reproduces the representative bootstrap statistics included in the supplement (replace with final numbers upon camera-ready).

\begin{table}[h]
\centering
\scriptsize                
\setlength{\tabcolsep}{4pt}
\renewcommand{\arraystretch}{0.95}
\caption{Representative bootstrap statistics for EVLM-RKTO.}
\label{tab:bootstrap_dummy}
\begin{tabular}{lcc}
\toprule
Evaluator & Mean Accuracy (\%) & 95\% Bootstrap CI (\%) \\
\midrule
Gemini Pro 1.5 & 95.4 & [94.2, 96.4] \\
LLAMA 405B & 94.8 & [93.5, 95.9] \\
GPT-4o & 96.2 & [95.1, 97.3] \\
Claude 3.5 Sonnet & 95.1 & [93.9, 96.2] \\
Human (majority) & 94.1 & [92.8, 95.4] \\
\midrule
Average (EVLM-RKTO) & 95.1 & [94.0, 96.0] \\
\bottomrule
\end{tabular}
\end{table}

\section{Theoretical Analysis}
\label{sec:notation}

\subsection{Notation and Batched RKTO Objective}

Here we restate the key symbols and provide the batched empirical form of the RKTO objective used during training.

\begin{itemize}[noitemsep]
  \item $x=(\mathcal{V},u)$: multimodal input context, consisting of visual inputs $\mathcal{V}$ and textual prompt $u$. 
        $\mathcal{V}$ may include an RGB image, video clip, depth map, or OCR-extracted text.
  \item $\hat y$: model-generated instruction tokens; 
        $\hat m$: predicted mask (decoded from mask tokens, when available);
        $m_{\mathrm{ref}}$: reference mask.
  \item $\rho_\phi(\cdot\mid x)$: current model (EVLM) conditional distribution, parameterized by $\phi$.
  \item $\rho_{\mathrm{ref}}(\cdot\mid x)$: reference policy (the SFT snapshot baseline).
  \item $\rho_{\mathrm{pref}}(\cdot\mid x)$: empirical human-preferred distribution.
  \item $y_{\mathrm{pref}}$: human-preferred response corresponding to input $x$.
  \item $R_{\mathrm{eff}}\!\in[0,1]$: instruction-effectiveness reward (semantic + mask IoU).
  \item $R_{\mathrm{reflect}}\!\in[0,1]$: reflection-quality reward.
  \item $w(\cdot)$: monotonic importance-weight function, defined as
        $w(s)=\mathrm{clip}(\mathrm{softplus}(s),0,w_{\max})$ with 
        $\mathrm{softplus}(s)=\log(1+e^s)$.
  \item $\lambda_{\mathrm{ref}}\!\ge0$: weighting coefficient for reflection alignment.
  \item $e(\cdot)$: Qwen2-VL text encoder embedding function; cosine values are rescaled to $[0,1]$ via $\widetilde{\cos}(a,b)=(1+\cos(a,b))/2$.
  \item Intersection-over-Union (IoU) between predicted and reference masks:
        \[
          \mathrm{IoU}(\hat m, m_{\mathrm{ref}})
          = \frac{|\hat m \cap m_{\mathrm{ref}}|}{|\hat m \cup m_{\mathrm{ref}}|}.
        \]
\end{itemize}

For a mini-batch $\{(x_i, y^{(i)}_{\mathrm{pref}})\}_{i=1}^B$, 
the normalized log-ratio statistics are computed as:
\begin{align}
\hat s_\phi^{(i)} &=
\log\frac{\rho_\phi\!\left(y^{(i)}_{\mathrm{pref}}\mid x_i\right)}
          {\rho_{\mathrm{ref}}\!\left(y^{(i)}_{\mathrm{pref}}\mid x_i\right)}, \\[2pt]
\hat\eta_0 &= 
\max\!\Biggl(0,\;
 \frac{1}{B(B-1)}\!
 \sum_{i\ne j}
 \log\frac{\rho_\phi\!\left(y^{(j)}_{\mathrm{pref}}\mid x_i\right)}
          {\rho_{\mathrm{ref}}\!\left(y^{(j)}_{\mathrm{pref}}\mid x_i\right)}\!
 \Biggr).
\end{align}

The empirical RKTO loss minimized during training is
\begin{equation}
\label{eq:batch_rkto_app}
\begin{split}
\widehat{\mathcal{L}}_{\mathrm{RKTO}}
= \frac{1}{B}\sum_{i=1}^B
  \Big[
    w\!\big(\hat s_\phi^{(i)}-\hat\eta_0\big)\,R^{(i)}_{\mathrm{eff}}
    + \lambda_{\mathrm{ref}}\,R^{(i)}_{\mathrm{reflect}}
  \Big].
\end{split}
\end{equation}

\subsection{Instruction Effectiveness $R_{\mathrm{eff}}$}
\label{sec:rewards}
The instruction-effectiveness reward $R_{\mathrm{eff}}\!\in\![0,1]$ measures how well the generated instruction $\hat y$ matches the human reference $y_{\mathrm{pref}}$ both semantically and spatially. It combines:
\begin{enumerate}[noitemsep]
  \item the cosine similarity between text embeddings of $\hat y$ and $y_{\mathrm{pref}}$, rescaled to $[0,1]$; and
  \item the Intersection-over-Union (IoU) between predicted and reference masks, also normalized.
\end{enumerate}

The final convex combination is
\begin{equation}
R_{\mathrm{eff}} = \alpha\,\widetilde{\mathrm{cos}}\!\big(e(\hat y), e(y_{\mathrm{pref}})\big)
                 + (1-\alpha)\,\mathrm{IoU}(\hat m,m_{\mathrm{ref}}),
\end{equation}
where $\alpha$ is tuned on a validation set ($\alpha{=}0.7$ in our experiments).

\subsection{Reflection Reward $R_{\mathrm{reflect}}$}
The reflection-quality reward encourages concise, semantically consistent self-reflection.  
We employ the differentiable surrogate:
\begin{align}
\label{eq:Rreflect_app}
R_{\mathrm{reflect}} &= 
\beta_1\,\widetilde{\mathrm{cos}}\!\big(e(r_{\mathrm{refl}}), e(y_{\mathrm{pref}})\big)
+ \beta_2\,\exp\!\big(-\gamma\,\mathrm{len}(r_{\mathrm{refl}})\big) \nonumber\\
&\quad+ \beta_3\!\Big(1-\mathrm{KL}\!\big(p_{\mathrm{int}}\|p_{\mathrm{refl}}\big)\Big),
\end{align}
where:
\begin{itemize}[noitemsep]
  \item $r_{\mathrm{refl}}$: reflection segment from the model’s CoT trace;
  \item $p_{\mathrm{int}}$, $p_{\mathrm{refl}}$: token distributions for the \texttt{<intermediate>} and \texttt{<reflection>} segments;
  \item $\mathrm{len}(\cdot)$: token length of reflection (brevity term);
  \item coefficients $\beta_1+\beta_2+\beta_3=1$ with $\beta_k\ge0$.
\end{itemize}

The differentiable reflection loss is 
$\mathcal{L}_{\mathrm{reflect}} = 1 - R_{\mathrm{reflect}}$.  
For non-differentiable components of $R_{\mathrm{eff}}$ (e.g., IoU),
we apply REINFORCE with a batch-mean baseline (see Sec.~\ref{sec:grad}).

\subsection{Gradient Estimator and Variance Reduction}
\label{sec:grad}

Differentiating Eq.~\eqref{eq:batch_rkto_app} yields the gradient estimator used in training.
For compactness, define 
$c_i = w(\hat s_\phi^{(i)}-\hat\eta_0)$ and 
$d_i = c_i\,(1-\alpha)\,(\mathrm{IoU}^{(i)}-b_{\mathrm{IoU}})$.
Then
\begin{equation}
\label{eq:grad_est_app}
\begin{aligned}
\nabla_\phi \widehat{\mathcal{L}}_{\mathrm{RKTO}}
\approx\;
&\frac{1}{B}\sum_{i=1}^{B} c_i\,R_{\mathrm{eff}}^{(i)}\,
  \nabla_\phi\log\rho_\phi\!\big(y^{(i)}_{\mathrm{pref}}\mid x_i\big)\\[3pt]
&+\lambda_{\mathrm{ref}}\nabla_\phi\mathcal{L}_{\mathrm{reflect}}\\[3pt]
&+\frac{1}{B}\sum_{i=1}^{B} d_i\,
  \nabla_\phi\log\rho_\phi\!\big(\hat m^{(i)}\mid x_i\big).
\end{aligned}
\end{equation}

Variance reduction techniques include:
(i) clipping $w(\cdot)$ to $[0,w_{\max}]$;  
(ii) centering rewards $(R-\bar R)$;  
(iii) smaller learning rate for RKTO ($\text{LR}_{\text{RKTO}}\!\ll\!\text{LR}_{\text{SFT}}$);  
(iv) averaging multiple Monte-Carlo mask samples per example; and  
(v) gradient and importance-weight clipping.

\subsection{Expected--KL Improvement}
\label{sec:theory}

The following theorem formalizes the expected improvement guarantee of the RKTO surrogate.

\begin{theorem}[Monotonic Alignment Guarantee]
\label{thm:monotone_app}
Let rewards $R_{\mathrm{eff}}$ and $R_{\mathrm{reflect}}$ be bounded in $[0,1]$,
$w(\cdot)$ be non-decreasing and bounded, and training steps sufficiently small.  
Define
\begin{align}
\mathcal{K}(\phi)
&=\mathrm{KL}\!\big(\rho_\phi(y\mid x)\,\|\,\rho_{\mathrm{pref}}(y\mid x)\big)\nonumber\\
&\quad+\lambda_{\mathrm{ref}}\,
\mathrm{KL}\!\big(\rho_\phi(r_{\mathrm{refl}}\mid x)\,\|\,\rho_{\mathrm{pref}}(r_{\mathrm{refl}}\mid x)\big).
\end{align}
Then any update step that decreases $\mathcal{L}_{\mathrm{RKTO}}$ in expectation
also decreases $\mathbb{E}_x[\mathcal{K}(\phi)]$, ensuring joint alignment of output and reflection distributions.
\end{theorem}

\begin{proof}
The gradient of each KL term can be expressed, via the log-derivative trick,
as an expectation of $\nabla_\phi\log\rho_\phi$ weighted by the log-density ratio.
The RKTO objective employs the same log-ratio statistic $s_\phi$ with baseline $\eta_0$
and monotonic weight $w(\cdot)$. Because the rewards are bounded and $w$ preserves
ordering, the update direction of $-\nabla_\phi\mathcal{L}_{\mathrm{RKTO}}$ is aligned
with that of $\nabla_\phi\mathcal{K}(\phi)$ in expectation.
Therefore, sufficiently small gradient steps that reduce $\mathcal{L}_{\mathrm{RKTO}}$
also reduce the expected composite divergence $\mathcal{K}(\phi)$, guaranteeing monotonic improvement.
\end{proof}

\section{Implementation Details}
\label{sec:impl}

Representative defaults used in experiments:
\begin{itemize}[noitemsep]
  \item Optimizer: AdamW. SFT LR $=5\times10^{-5}$; RKTO LR $=2\times10^{-5}$.
  \item Batch size $B=64$ (use gradient accumulation if needed).
  \item $\lambda_{\mathrm{ref}}=0.2$, $\beta=(0.5,0.2,0.3)$, $\gamma=0.2$.
  \item $w_{\max}=10$, softplus threshold $\tau=0.1$.
  \item Monte-Carlo samples per example: 3 (empirical trade-off between variance and cost).
  \item Reference policy $\rho_{\mathrm{ref}}$: SFT snapshot (pre-compute/caching of log-probs when feasible).
  \item LoRA rank 64, alpha 16, dropout 0.05 (when LoRA used).
\end{itemize}

Compute usage (approx):
\begin{itemize}[noitemsep]
  \item SFT (30k samples): $\sim$48 GPU-hours on 8$\times$A100 (80GB).
  \item RKTO stage: $\sim$36 GPU-hours on 8$\times$A100 (80GB).
\end{itemize}

\section{Dataset details and prompts}
\label{sec:dataset}
\begin{table}[t]
    \centering
    \caption{Data Distribution for EVLM CoT Fine-Tuning}
    \label{tab:data_distribution}
    \begin{tabular}{p{0.4\columnwidth} p{0.2\columnwidth} c}
        \hline
        \textbf{Reference} & \textbf{Original} & \textbf{\# of Samples} \\
        \hline
        Image & Image & 10,000 \\
        Image+Text & Image & 5,000 \\
        Video & Image & 4,000 \\
        Image & Video & 3,000 \\
        Video+Text & Image & 3,000 \\
        Text & - & 5,000 \\
        \hline
    \end{tabular}\label{tab:dataset}
\end{table}
\begin{figure}[t] \centering \includegraphics[width=0.8\linewidth]{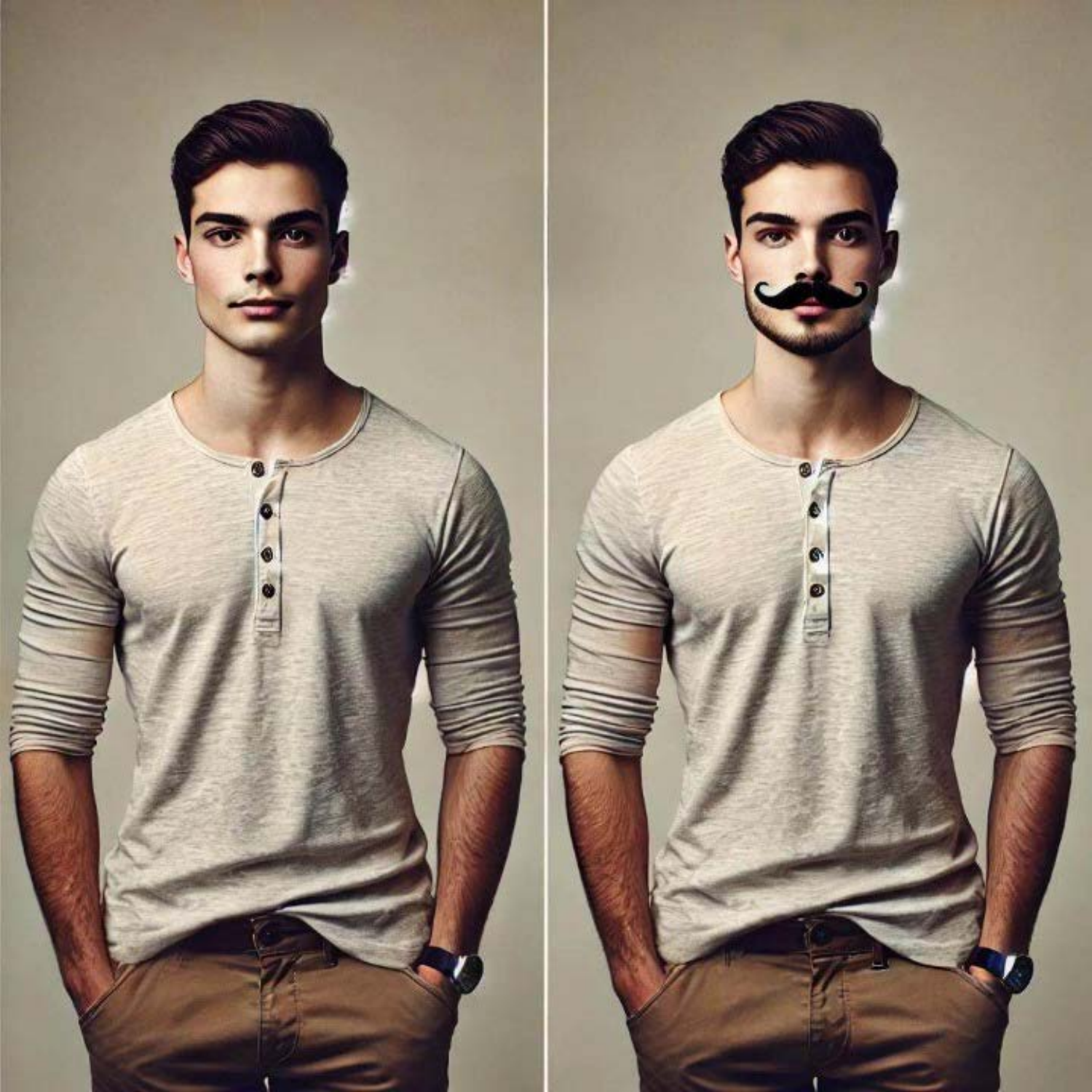} \caption{An example of a diptych generated for the EVLM dataset using DALL-E-3. The left panel displays the original image, while the right panel showcases the edited image, reflecting the instruction to add a mustache.} \label{fig:diptych_example} 
\end{figure}
To train EVLM for multimodal editing instruction generation, we developed a comprehensive dataset that emphasizes diversity in input modalities and output clarity. The dataset preparation process involved \textbf{generation}, \textbf{augmentation}, and \textbf{refinement with human evaluation}, ensuring both the quality and breadth of data. In our training, we also include nearly 20\% training samples from cross-domain datasets used in \cite{cheng2024vision,chen2024internvl,xu2024llava} in addition to our MME data. 
\noindent Cross-benchmark evaluations in the main paper use:
\begin{itemize}[noitemsep]
  \item MMMU / MMMU-Pro \citep{yue2024mmmu}
  \item MathVista \citep{lu2023mathvista}
  \item ChartQA \citep{masry2022chartqa}
  \item AI2 Diagram \citep{kembhavi2016diagram}
  \item DocVQA \citep{mathew2021docvqa}
  \item VQAv2 \citep{shen2023generative}
\end{itemize}
\subsubsection{Dataset Overview}
The dataset consists of 30,000 samples distributed across six categories of input combinations, as shown in Table~\ref{tab:dataset}. Each sample contains:
\begin{itemize}
    \item \textbf{Reference Input (\(R\))}: A combination of image, video, or text to provide editing context.
    \item \textbf{Original Input (\(O\))}: The target image or video requiring modifications.
    \item \textbf{Reflective Rationale and Instruction}: A detailed rationale generated by GPT-4o, followed by a human-reviewed editing instruction.
\end{itemize}

\begin{table*}[htb!]
\centering
\caption{Examples of Reference Images and Corresponding Editing Instructions}
\label{tab:editing_examples}
\begin{tabular}{c|c|p{7cm}}
\hline
\textbf{Reference Image 1} & \textbf{Reference Image 2} & \textbf{Editing Instruction} \\ \hline
Einstein's face & Green threads & Turn his face into Einstein and turn his jacket green. \\ \hline
Batman’s face & Blue shirt & Turn his face into Batman and turn his shirt blue. \\ \hline
Spider-Man’s mask & Leather jacket & Turn his face into Spider-Man’s mask and turn his jacket into leather. \\ \hline
Mona Lisa’s face & Green dress & Turn her face into Mona Lisa’s face and turn her dress green. \\ \hline
Robotic face & Black boots & Turn his face into a robotic face and turn his boots black. \\ \hline
Cat's face & Yellow scarf & Turn their face into a cat's face and turn their scarf yellow. \\ \hline
Superman's face & Red trousers & Turn his face into Superman’s face and turn his trousers red. \\ \hline
A clown's face & Striped T-shirt & Turn his face into a clown's face and turn his T-shirt into a striped pattern. \\ \hline
Iron Man’s mask & Metallic armor & Turn his face into Iron Man’s mask and turn his outfit into metallic armor. \\ \hline
A Viking’s face & Fur coat & Turn his face into a Viking’s face and turn his coat into fur. \\ \hline
A painter’s face & Multi-colored apron & Turn her face into a painter’s face and turn her apron multi-colored. \\ \hline
Santa Claus’s face & Red and white suit & Turn his face into Santa Claus and turn his outfit into a red and white suit. \\ \hline
A robot’s head & Silver gloves & Turn his face into a robot’s head and turn his gloves silver. \\ \hline
A pirate’s face & Black hat & Turn his face into a pirate’s face and add a black hat. \\ \hline
A magician’s face & White gloves & Turn his face into a magician’s face and turn his gloves white. \\ \hline
A superhero’s face & Cape & Turn his face into a superhero’s face and add a red cape. \\ \hline
A lion’s face & Golden mane & Turn their face into a lion’s face and turn their hairstyle into a golden mane. \\ \hline
A wizard’s face & Magic staff & Turn his face into a wizard’s face and add a magic staff to his hand. \\ \hline
A medieval knight’s face & Shiny armor & Turn his face into a medieval knight’s face and turn his outfit into shiny armor. \\ \hline
\end{tabular}
\end{table*}

\begin{table*}[t]
\centering
\caption{Examples of Ambiguous Text and Image as References for Editing Instructions}
\label{tab:ambiguous_references}
\begin{tabular}{c|c|p{7cm}}
\hline
\textbf{Reference 1 (Text)} & \textbf{Reference 2 (Image)} & \textbf{Editing Instruction} \\ \hline
Make it heroic & Superman's costume & Turn his outfit into Superman's costume while keeping his face unchanged. \\ \hline
Give it a metallic vibe & A shiny silver texture & Turn the jacket into a silver metallic style while preserving the rest of the image. \\ \hline
Bring a touch of royalty & A golden crown & Add a golden crown to the person's head without altering their facial features. \\ \hline
Add an artistic feel & Van Gogh’s painting & Turn his jacket into a style resembling Van Gogh's Starry Night painting. \\ \hline
Transform into a character & Spider-Man’s suit & Change his outfit into Spider-Man’s suit while leaving his face intact. \\ \hline
Make it winter-ready & A fur coat & Replace his jacket with a warm fur coat suitable for winter. \\ \hline
Add a festive spirit & Christmas decorations & Turn her dress into a Christmas-themed outfit with red and white patterns. \\ \hline
Create a futuristic look & A robotic arm design & Replace his arms with robotic prosthetics while keeping the rest of the image unaltered. \\ \hline
Make it classic & A vintage tuxedo & Turn his outfit into a vintage tuxedo, keeping his hairstyle and face the same. \\ \hline
Turn it into nature & A green leafy texture & Change her jacket into a pattern resembling green leaves. \\ \hline
Add a magical touch & A wizard's robe & Transform his outfit into a wizard’s robe with stars and moons. \\ \hline
Show some adventure & A pirate's hat & Add a pirate’s hat and an eyepatch while leaving the outfit unchanged. \\ \hline
Brighten it up & A vibrant yellow scarf & Add a yellow scarf to her outfit without altering any other details. \\ \hline
Make it sporty & A football jersey & Replace his shirt with a football jersey representing a famous team. \\ \hline
Bring a cultural touch & A traditional Japanese kimono & Turn her outfit into a traditional Japanese kimono. \\ \hline
\end{tabular}
\end{table*}
 \begin{figure*}[t]
    \centering
    \includegraphics[width=0.9\linewidth]{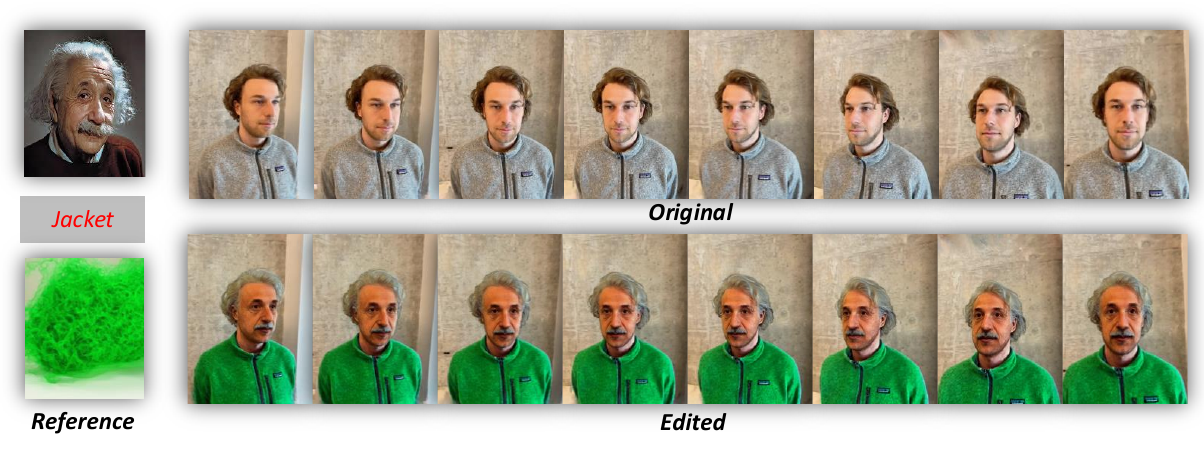}

    \caption{\textbf{Reference Image and Text.} Examples of 3D editing results generated using EVLM + IP2P. The model produces context-aware editing instructions based on the reference content, autonomously interpreting the editing rationale to generate optimal instructions for the desired outcome. This example demonstrates multi-attribute editing, where the jacket is transformed into green, and the face is altered to resemble Einstein.}

    \label{fig:3D_results}

\end{figure*}
\begin{figure}[htb]

    \centering
    \includegraphics[width=1\linewidth, clip]{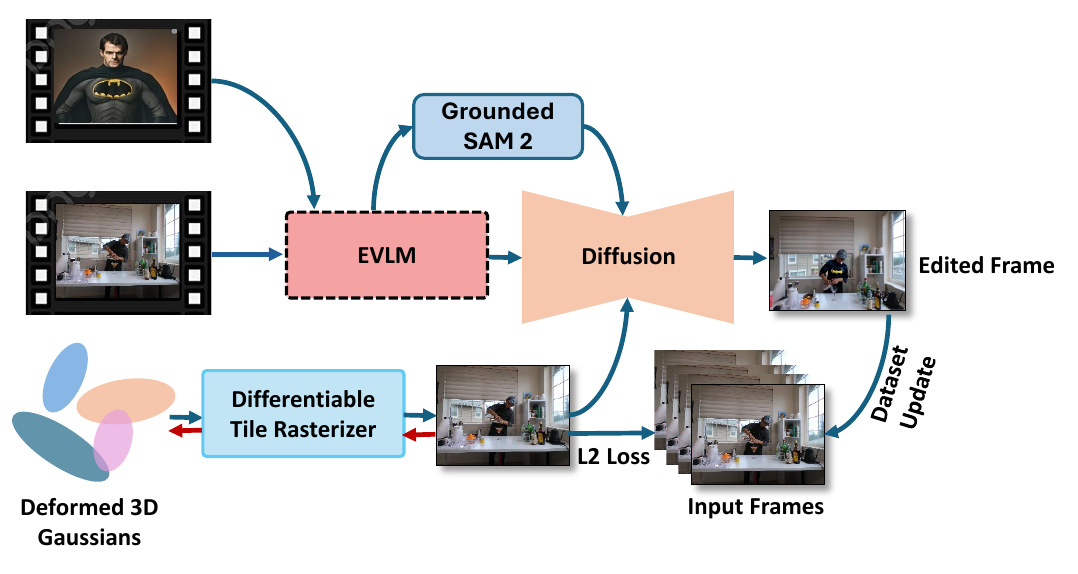}
    \caption{\footnotesize The 4D editing pipeline guided by EVLM, which outputs both the editing instruction and the target object to be modified. The identified target object is passed through Grounded SAM 2 to generate precise masks, which guide local editing in the diffusion model based on the provided instruction. The edited frames are then iteratively updated in the dataset to ensure temporal and spatial consistency, while an L2 loss is applied between input and edited frames for quality refinement, resulting in a consistent and temporally smooth 4D dataset.}
    \label{fig:EVLM+4D}
\end{figure}

\noindent \paragraph{Generation of Paired Images} To prepare a subset where both the reference and the original are images, we generate paired image samples using DALL-E 3. Since we want EVLM to understand the relationship between the given reference and original visual cues, instead of relying solely on textual descriptions, we create \textbf{diptychs}, where the original input image and the edited output image are displayed side by side. This approach enhances alignment and correspondence between input-output image pairs by ensuring both visual and semantic consistency.

We use prompts designed to capture various transformations while maintaining the visual integrity of the input image. For instance, to generate the paired image from DALL-E 3 shown in Figure~\ref{fig:diptych_example}, we use a prompt like:
\begin{quote} \textit{Generate a diptych with two side-by-side images. On the left side, there should be a person standing confidently with a neutral expression, wearing casual clothing. On the right side, show the same person with a mustache while keeping the rest of the visual attributes unchanged.} \end{quote}

This methodology allows for generating a variety of image pairs that focus on specific edits, such as adding or removing objects, changing facial features, or modifying global attributes like color or style. By leveraging DALL-E 3’s generative capabilities, we ensure that the resulting diptychs are coherent and contextually aligned, providing a strong foundation for fine-tuning EVLM.

Figure~\ref{fig:diptych_example} illustrates an example of such a diptych. The left panel represents the reference image, while the right panel reflects the transformation specified by the prompt. This technique is particularly effective when working with samples where textual input is minimal or absent, as it allows EVLM to focus on interpreting visual relationships directly. We have furnished some examples in Table~\label{tab:editing_examples} where we use two reference images.

Overall, the diptych approach not only simplifies the editing process but also improves the model's understanding of contextual relationships between the original and modified visual content. By avoiding dependency on detailed textual descriptions, this method provides a streamlined and efficient pathway for multimodal image editing.

    \begin{figure}[htb!] \centering \includegraphics[width=0.8\linewidth]{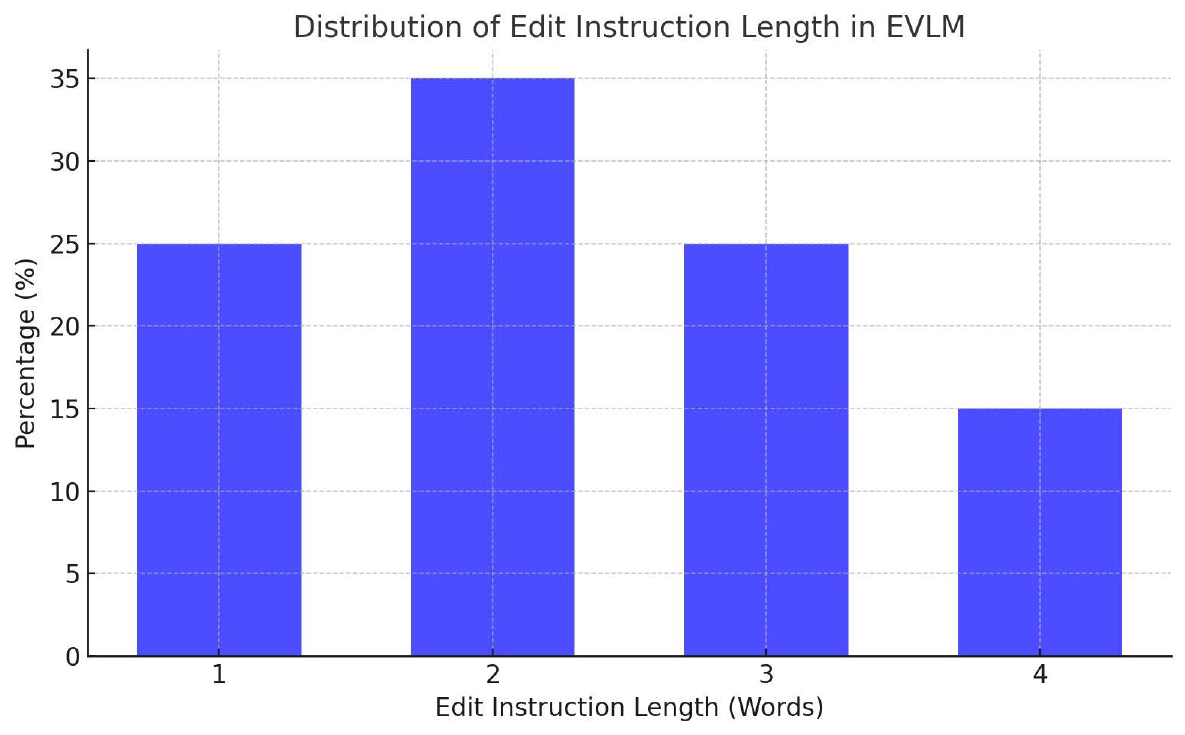}
    \caption{EVLM effectively performs editing tasks using concise instructions (1--4 words), whereas InstructPix2Pix requires longer and more detailed instructions (4 or more words).}
    \label{fig:edit_length_comparison}
\end{figure}

In addition to using DALL-E 3 for generating paired images, we further leverage Stable Diffusion XL to create a diverse set of high-quality image samples. These generated images undergo pre-processing to ensure compatibility and are combined with other open-source datasets such as IN2N, 3DEgo, DyNeRF, and MSCOCO~\cite{lin2014microsoft} to form reference-original pairs. This combination enriches the dataset with varied visual transformations and styles, enhancing the robustness of EVLM. For video-based pairs, we employ generative video models like CogVideoX-5B~\cite{yang2024cogvideox} and Stable Video Diffusion~\cite{blattmann2023stable} to synthesize reference videos, while the original videos are sourced from publicly available datasets, including Kinetics~\cite{kay2017kinetics} and Charades~\cite{sigurdsson2018charades}. This multi-source, multi-modality data preparation process ensures that the training data encapsulates a wide range of visual edits and transformations, providing EVLM with the capability to generalize across diverse scenarios.

\noindent \paragraph{Reference Text.} Figure~\ref{fig:edit_length_comparison} illustrates the distribution of instruction (Reference Text) lengths given as an input to EVLM. 8000/30000 samples have the reference text in addition to reference image or a video. EVLM's instructions, typically ranging from 1--4 words, demonstrate the model's ability to perform complex editing tasks with minimal and concise textual input, unlike InstructPix2Pix, which relies on more detailed instructions exceeding 4 words.
Figure~\ref{fig:edit_distribution} provides a detailed view of the types of edits and their frequency, showcasing the dataset's balanced coverage of diverse editing tasks.
We additionally created samples incorporating both textual and visual references, as illustrated in Table~\ref{tab:ambiguous_references}.

\begin{figure}[H]
    \centering
    \includegraphics[width=0.8\linewidth]{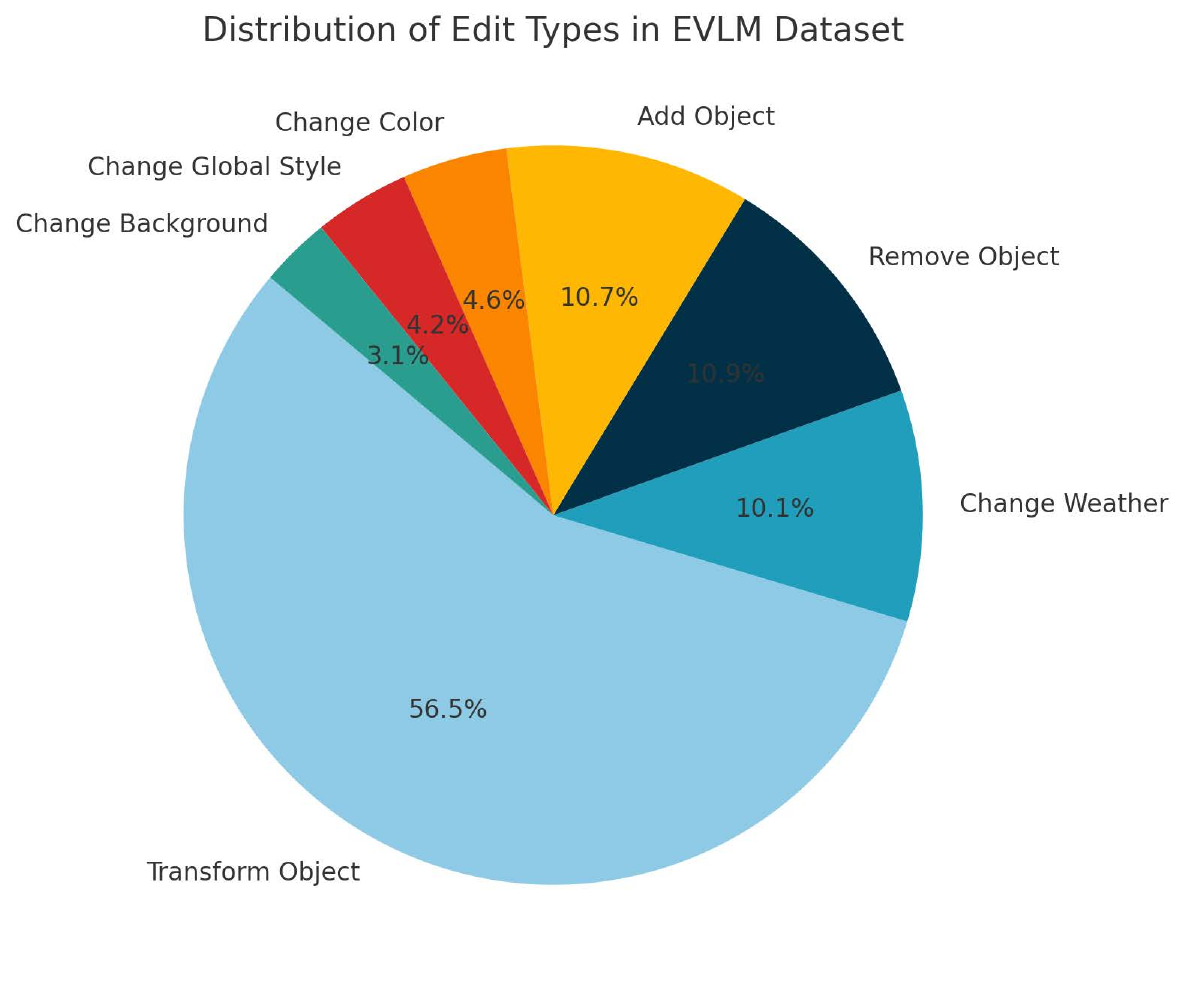}
    \caption{Distribution of Edit Types in the EVLM Dataset. The dataset covers diverse editing tasks, with the majority focusing on object transformations.}
    \label{fig:edit_distribution}
\end{figure}

\subsubsection{Data Quality Assessment}
To ensure high data quality, the dataset underwent the following refinements:
\begin{itemize}
    \item \textbf{Human Evaluation}: Human raters validated and refined GPT-4o-generated rationales and instructions to align with subjective editing preferences.
    \item \textbf{Instruction Clarity}: Editing instructions were enriched to include sufficient detail, distinguishing them from the vague instructions typically used in baseline models.
    \item \textbf{Balanced Task Representation}: The dataset was carefully curated to balance global edits (e.g., changing weather or background) and local edits (e.g., object transformation or removal), as illustrated in Figure~\ref{fig:edit_distribution}.
\end{itemize}

\subsubsection{Edit Diversity}
The dataset spans a wide range of editing tasks, encompassing both global and local transformations. As shown in Figure~\ref{fig:edit_distribution}, the majority of tasks involve object transformations (56.5\%), followed by weather changes (10.1\%), object removals (10.9\%), and object additions (10.7\%). The diversity of edit types ensures that EVLM can generalize across varied scenarios.



\noindent By integrating reflective reasoning, human evaluation, and diverse multimodal inputs, this dataset provides a robust foundation for training EVLM to generate accurate and human-aligned editing instructions.

\section{Additional Results}
\label{sec:add_results}

The figure~\ref{fig:EVLM+4D} demonstrates how EVLM processes the reference video and original video to generate context-aware editing instructions and execute the desired transformations. In this specific 4D editing example, the goal is to change the person’s clothing into a Batman costume while preserving the face, as the reference video does not depict a Batman face. EVLM interprets the required editing from the reference and original videos, outputs the object label (clothing), and passes it to Grounded-SAM, which generates precise masks for the targeted region. These masks ensure that only the clothing is affected during the editing process. The editing instruction is then fed into a diffusion model, which applies the required transformation to the clothing while maintaining the integrity of other visual attributes, such as the face. The edited frames are iteratively replaced in the original dataset, creating a modified dataset aligned with the desired edits. Training is continued on this edited dataset to refine the alignment of deformed 3D Gaussians with the edited frames, ensuring the model's enhanced capacity to generalize and align with complex transformations. 

For 3D editing, we do not utilize deformed Gaussians; instead, we employ the standard Gaussian Splatting technique~\cite{kerbl20233d}. Along with this supplementary document, we have provided the rendered videos for both 4D-editing and 3D-editing. Additionally, Figure~\ref{fig:3D_results} showcases a 3D editing example where the input to EVLM consists of four components: two images (Einstein and green threads), one text prompt (jacket), and multiple frames of a person's face. The model autonomously interprets the editing intent and generates the instruction: \textit{Turn his face into Einstein and turn his jacket into green}. This instruction can then be utilized with~\cite{brooks2023instructpix2pix} to perform the desired edits, as illustrated in the Figure \ref{fig:3D_results}.

\section{Ablation Studies}
\label{sec:ablations}

We analyze the sensitivity of EVLM to its main hyperparameters.

\subsection{Effect of reflection weight $\lambda_{\mathrm{ref}}$}
As shown in Table~\ref{tab:abl_lambda}, increasing $\lambda_{\mathrm{ref}}$ strengthens reflective reasoning until $\lambda_{\mathrm{ref}}\!>\!0.2$, where excessive emphasis on reflection marginally reduces instruction fidelity.

\begin{table}[h]
\centering
\small
\caption{Ablation on $\lambda_{\mathrm{ref}}$.}
\label{tab:abl_lambda}
\begin{tabular}{lccc}
\toprule
$\lambda_{\mathrm{ref}}$ & Accuracy (\%) & Reflection Q & IoU \\
\midrule
0.00 & 88.0 & 0.62 & 0.41 \\
0.05 & 90.5 & 0.68 & 0.46 \\
0.10 & 92.3 & 0.72 & 0.51 \\
0.20 & \textbf{94.1} & 0.78 & 0.58 \\
0.50 & 93.4 & 0.80 & 0.56 \\
\bottomrule
\end{tabular}
\end{table}

\subsection{Importance-weight cap $w_{\max}$}
Smaller caps stabilize training but slow convergence; larger caps risk variance spikes.  
Table \ref{tab:abl_wmax} shows $w_{\max}{=}10$ achieves the best trade-off.

\begin{table}[H]
\centering
\small
\caption{Ablation on $w_{\max}$.}
\label{tab:abl_wmax}
\begin{tabular}{lcc}
\toprule
$w_{\max}$ & Accuracy (\%) & Comment \\
\midrule
5 & 92.1 & conservative weighting \\
10 & \textbf{94.1} & default, stable \\
20 & 93.6 & slightly higher variance \\
\bottomrule
\end{tabular}
\end{table}
\begin{table}[H]
\centering
\small
\caption{Ablation on MC samples for importance-weight estimation.}
\label{tab:abl_mc}
\begin{tabular}{lcc}
\toprule
MC samples & Accuracy (\%) & Runtime × factor \\
\midrule
1 & 92.8 & 1.0 \\
2 & 93.9 & 1.3 \\
3 & \textbf{94.1} & 1.6 (default) \\
4 & 94.2 & 2.0 \\
8 & 94.3 & 3.8 \\
\bottomrule
\end{tabular}
\end{table}
\begin{table}[H]
\centering
\small
\caption{Surrogate vs REINFORCE for reflection reward.}
\label{tab:abl_reflect}
\begin{tabular}{lcc}
\toprule
Method & Accuracy (\%) & Note \\
\midrule
Surrogate-only & 93.8 & lower variance \\
Surrogate + REINFORCE (IoU) & \textbf{94.1} & best overall \\
REINFORCE-only & 92.9 & unstable \\
\bottomrule
\end{tabular}
\end{table}
\subsection{Monte-Carlo samples per example}
Table \ref{tab:abl_mc} shows that three samples strike an effective balance between accuracy and runtime.

\subsection{Reflection reward formulation}
Replacing the differentiable surrogate with REINFORCE increases variance and slows convergence.  
Combining both retains stability while preserving gradient signal for mask IoU (Table \ref{tab:abl_reflect}).

\subsection{Alternative objectives}
We further compared pure SFT, KTO, and RKTO (Table \ref{tab:kto_sft_comparison}).  
Results confirm that reflection-aware optimization yields the strongest alignment.

\begin{table}[H]
\centering
\footnotesize
\caption{Training-objective comparison (human evaluation).}
\label{tab:kto_sft_comparison}
\begin{tabular}{lc}
\toprule
Method & Accuracy (\%) \\
\midrule
EVLM-KTO & 75.1 \\
EVLM-KTO (2-step thinking) & 76.3 \\
EVLM-SFT & 67.2 \\
EVLM-SFT (2-step thinking) & 65.1 \\
\midrule
\textbf{EVLM-RKTO} & \textbf{94.1} \\
\bottomrule
\end{tabular}
\end{table}

\paragraph{Discussion.}
Across all ablations, the reflection weight $\lambda_{\mathrm{ref}}$ and surrogate-based reflection reward have the largest impact.  
Moderate reflection emphasis (0.2) ensures coherent reasoning without over-regularizing the generation policy.  
RKTO thus balances reasoning depth and instruction fidelity, producing interpretable and precise edits.
\begin{table}[H]
\centering
\small
\caption{Contextual alignment from user study.}
\label{tab:contextual_alignment}
\begin{tabular}{lcc}
\toprule
Editing Type & EVLM & Best Baseline \\
\midrule
2D Editing & \textbf{4.82} & IP2P: 4.32 \\
Video Editing & \textbf{4.77} & Tune-A-Video: 4.41 \\
3D Editing & \textbf{4.65} & IN2N: 4.38 \\
4D Editing & \textbf{4.58} & Instruct 4D-to-4D: 4.42 \\
\bottomrule
\end{tabular}
\end{table}


\section{User study}
\label{sec:user}

To evaluate the contextual alignment of EVLM compared to other methods, we conducted a user study involving 50 participants aged between 20 and 40. Each participant evaluated the contextual alignment of editing outputs on a scale of 1 to 5 across four editing types: 2D, video, 3D, and 4D editing. Participants were presented with pairs of reference inputs (text, images, or videos) and their corresponding edited outputs and were tasked with rating how well the edits aligned with the intended context. The study involved a total of 200 edits, with 50 edits per task, and 20 randomly selected edits were assigned to each participant for evaluation.

The Table~\ref{tab:contextual_alignment} presents the average scores across all users for each method. EVLM consistently outperformed competing methods (IP2P for 2D editing, TAV for video editing, IN2N for 3D editing, and Instruct 4D-to-4D for 4D editing) across all editing types. However, no method achieved a perfect 5.0 score, reflecting the subjective nature of the task and variability in user preferences. The results demonstrate EVLM's superior ability to understand and apply contextual cues, particularly in complex multimodal scenarios like 3D and 4D editing.

\section{Limitations }
\label{sec:limitations}

While EVLM demonstrates significant advancements in generating context-aware editing instructions, it has several limitations. The model heavily relies on the quality and diversity of its training dataset, which, despite comprising 30,000 samples, may not generalize well to highly complex or unseen scenarios. Its dependency on human-rated rationales introduces subjectivity, and its computationally intensive reflective reasoning framework and KL-Divergence Target Optimization (KTO) limit scalability for large-scale or real-time applications. Furthermore, EVLM struggles with ambiguous textual references and conflicting multimodal inputs, and it primarily supports object-specific edits, making it less effective for global scene-level transformations. The precision of mask generation by Grounded-SAM impacts the quality of detailed edits, and the lack of dynamic or incremental training reduces adaptability to new styles or user preferences. Additionally, EVLM is restricted to text and vision modalities, excluding other forms of input such as audio, which limits its applicability in broader multimodal contexts. These limitations provide opportunities for future improvements in dataset diversity, scalability, and multimodal adaptability.

\end{document}